\documentclass[pdflatex,sn-mathphys-num]{sn-jnl}

\usepackage{graphicx}%
\usepackage{multirow}%
\usepackage{amsmath,amssymb,amsfonts}%
\usepackage{amsthm}%
\usepackage{mathrsfs}%
\usepackage[title]{appendix}%
\usepackage{xcolor}%
\usepackage{textcomp}%
\usepackage{manyfoot}%
\usepackage{booktabs}%
\usepackage{listings}%

\usepackage{enumitem}%
\usepackage{float}%
\usepackage{hyperref}%
\usepackage[ruled,linesnumbered,vlined]{algorithm2e}%
\usepackage{mathtools}%
\usepackage{booktabs}%
\usepackage{multirow}%
\usepackage{adjustbox}%
\usepackage{verbatim}%
\usepackage{url}%
\usepackage{array}%
\usepackage{textcomp}%
\usepackage{stfloats}%
\usepackage{caption}%
\usepackage{subcaption}%

\newcommand{\etal}[0]{\textit{et al.}}
\theoremstyle{plain}
\newtheorem{defn}{Definition}
\newtheorem{property}{Property}
\newtheorem{proposition}{Proposition}

\theoremstyle{definition}

\newtheorem{example}{Example}

\DeclarePairedDelimiter\abs{\lvert}{\rvert}%
\begin{document}

\title[Article Title]{Multi-User Personalisation in Human-Robot Interaction: Resolving Preference Conflicts Using Gradual Argumentation}


\author[1]{\fnm{Aniol} \sur{Civit}}\email{acivit@iri.upc.edu}

\author[1]{\fnm{Antonio} \sur{Andriella}}\email{aandriella@iri.upc.edu}

\author[2]{\fnm{Carles} \sur{Sierra}}\email{sierra@iiia.csic.es}

\author[1]{\fnm{Guillem} \sur{Alenyà}}\email{galenya@iri.upc.edu}

\affil[1]{\orgdiv{Institut de Robòtica i Informàtica Industrial}, \orgname{CSIC-UPC}, \orgaddress{\street{Llorens i Artigas 4-6}, \city{Barcelona}, \postcode{08028}, \country{Spain}}}

\affil[2]{\orgdiv{Institut d'Investigació en Intel·ligència Artificial}, \orgname{IIIA-CSIC}, \orgaddress{\street{Campus de la UAB}, \city{Cerdanyola del Vallès}, \postcode{08193}, \country{Spain}}}


\abstract{
    While personalisation in Human-Robot Interaction (HRI) has advanced significantly, most existing approaches focus on single-user adaptation, overlooking scenarios involving multiple stakeholders with potentially conflicting preferences.
    To address this, we propose the Multi-User Preferences Quantitative Bipolar Argumentation Framework (MUP-QBAF), a novel multi-user personalisation framework that extends Quantitative Bipolar Argumentation Frameworks (QBAFs) to model and resolve preference conflicts. Unlike prior work in Argumentation Frameworks, which typically assumes static inputs, our approach incorporates the users' arguments and the robot’s dynamic observations of the environment, allowing the system to continuously adapt and respond to changing contexts. Preferences are represented as arguments whose strength is recalculated iteratively based on new information.
    The framework’s properties and capabilities are presented and validated through a realistic case study, where an assistive robot mediates between the conflicting preferences of a caregiver and a care recipient during a frailty assessment task. This evaluation includes a sensitivity analysis of argument base scores, demonstrating how preference outcomes can be shaped by user input and contextual observations. By offering a transparent, structured, and context-sensitive approach to resolving competing user preferences, this work advances the field of multi-user HRI, providing a principled alternative to data-driven methods, enabling robots to navigate conflicts in real-world environments.}




\maketitle

\section{Introduction}

    The effectiveness of personal robots largely depends on their ability to autonomously adjust to the diverse needs of individual users~\cite{Yang_hri_ws23}. This requires considering various factors, such as cultural background, personal preferences, and both cognitive and physical capabilities~\cite{Gasteiger_IJSR23}. Personalisation emerges as a key requirement for enhancing Human-Robot Interaction (HRI) in real-world settings as it has proven to improve engagement and foster trust and rapport~\cite{Smarr_HFESAM12, Di_Napoli_UMUAI23, Hofstede_FRAI25}.
    
    Most previous research has primarily focused on addressing preferences from a single user~\cite{Moro_THRI18, Park_aaai19, Falerni_RCIM24}. However, in many robotic contexts, preferences may also arise from multiple stakeholders involved in the usage of the technology. For instance, in healthcare, one of the most relevant use cases for assistive robotics, robots must account for not only patients' preferences but also those of therapists and doctors~\cite{Trainum_hri24}. In existing approaches, caregivers' preferences and knowledge are typically integrated into the robot's decision-making a priori, often through co-design sessions, which may include the care recipients~\cite{Bardaro_IJSR22, Gebelli_roman24}. However, scenarios where robots can autonomously decide while considering competing preferences of multiple users remain largely unexplored~\cite{Bernatskiy_alife14, Civit_ceurws24}. Furthermore, in these scenarios, the preferences that robots must consider include not only the \textit{positive} preferences (what the agents desire) but also the \textit{negative} preferences (what the agents reject)~\cite{Benferhat_kr02}, which have been largely overlooked in HRI research. 

    \begin{figure}[tbp!]
        \centering
        \includegraphics[width=1\columnwidth]{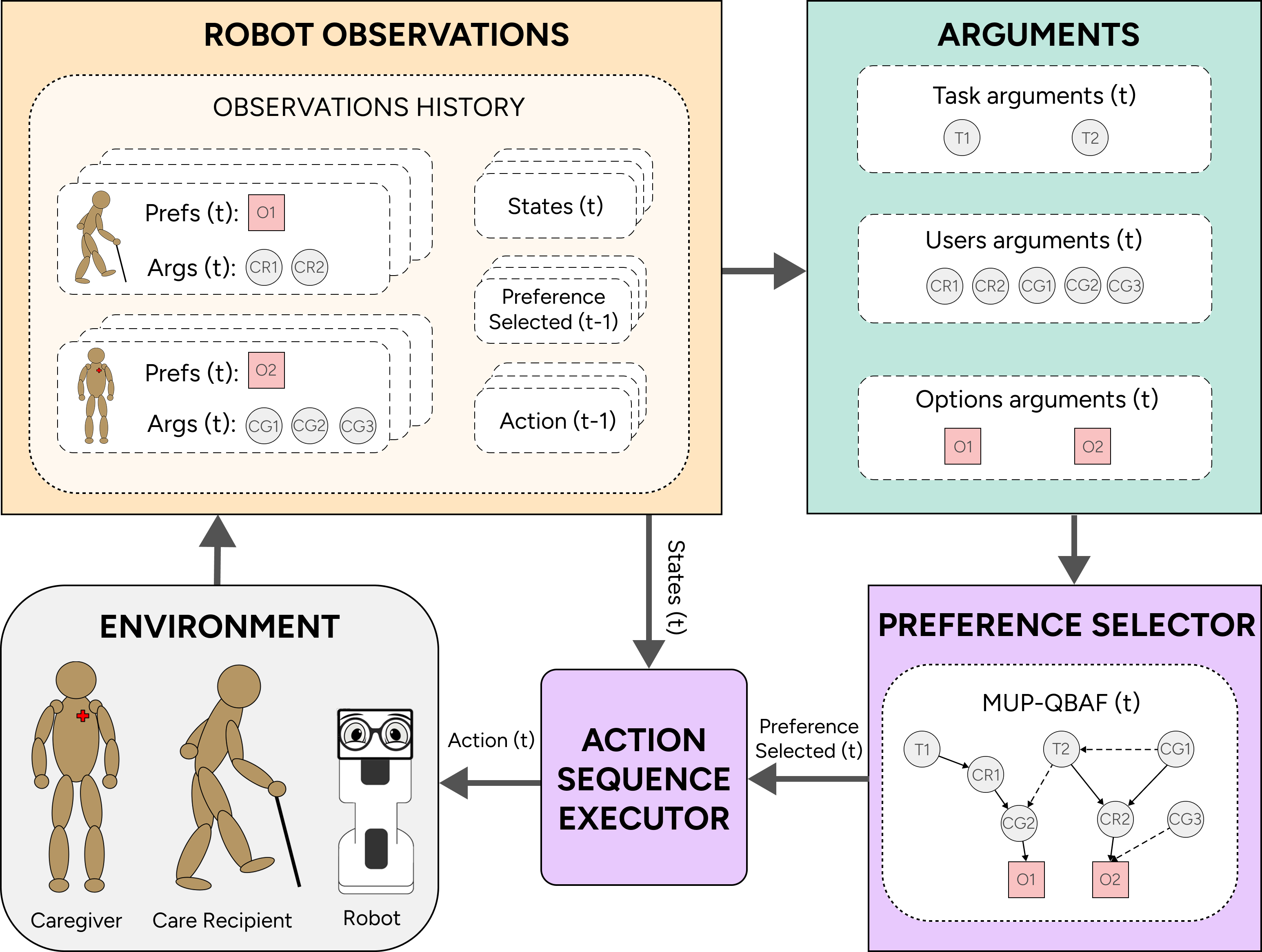}
        \caption{Iterative closed-loop multi-user preference selection system. The robot processes observations (\textbf{Robot Observations Module}), such as environmental states and user inputs to generate tasks and users' arguments, and the possible decisions (their preferences), which are stored in the \textbf{Arguments Module}. These are structured in the \textbf{Preference Selector Module}, where the arguments are combined into support/attack relations. The \textbf{MUP-QBAF} (Argumentation Framework) outputs a selected preference that directs robot action. Importantly, the environmental feedback from this action generates new observations, iteratively updating the argumentation framework.}
        \label{fig:robot_decision_intro}
    \end{figure}

    Addressing these multifaceted challenges, integrating competing, positive, and negative preferences from multiple stakeholders, demands a computational framework that can explicitly represent and reason about conflict in a transparent manner. While data-driven approaches, such as machine learning, have proven effective for learning individual user preferences~\cite{Andriella_UMUAI25}, they present several limitations in this context. Firstly, their ``black-box'' nature makes it difficult to explain why a particular decision was made, which is a critical factor for establishing trust and acceptance in HRI~\cite{Sado_ACM23, Love_roman24}. Secondly, these models handle conflict implicitly by learning a statistical mapping, rather than by explicitly reasoning about the arguments for and against a course of action~\cite{Macagno_18}. Finally, they often require substantial amounts of interaction data, making them less suitable for scenarios with new users or dynamically changing preferences~\cite{Carroll_arxiv24}.

    Computational Argumentation (CA) offers a promising approach to addressing this challenge. Specifically, it offers a structured method for resolving conflicts by evaluating competing reasons. To address this, CA employs Argumentation Frameworks (AFs), in which a set of arguments (representing reasons for or against a decision) is connected through relational structures, where some arguments may attack or support others~\cite{Amgoud_AI09}. By explicitly modelling users’ preferences as arguments and analysing their relationships, the system can determine the most reasonable decision, even in cases of disagreement.

    This work introduces the Multi-User Preferences Quantitative Bipolar Argumentation Frameworks (MUP-QBAFs), a \textbf{novel framework that enables robots to resolve conflicts when multiple users express competing preferences}. The proposed method extends Quantitative Bipolar Argumentation Frameworks (QBAFs) to select a preference by evaluating users' stated preferences, their underlying reasons, and dynamic environmental observations, all modelled as arguments within the framework. This gradual argumentation approach enables a possible fine-tuning of users' importance of each argument, and avoids the current limitation in abstract argumentation, which is the loss of information due to the framework reductions~\cite{Battaglia_icsum24}. The selected preference then parameterises the robot's decision-making system. 
    This framework offers additional advantages for HRI. Firstly, a key property of AFs is their representation as a graph, where the arguments are the nodes and the edges are their relations, producing transparent decisions~\cite{Rago_kr20}, unlike opaque learning-based methods. Transparency in the decision-making process is fundamental to providing explanations to users, which is crucial in real-world HRI. Secondly, AFs are inherently adaptable; arguments can be added or modified dynamically, enabling immediate decision updates without retraining, making them suitable for long-term interactions where preferences evolve over time~\cite{Irfan_UMUAI23}. Finally, while most of the literature in AFs assumes static decisions, HRI introduces a unique requirement: iterative decision-making, that is, making the same decision several times during an interaction. The outcome of the decision can change from one iteration to another due to the actions of the robot, new observations, or new arguments provided by the users~\cite{Ferretti_AI17}.
    These dynamic properties are defined and demonstrated through a case study, in which an assistive robot must reconcile conflicting preferences between a caregiver and a care recipient during an assistive task. Finally, the possible outcomes of the framework are analysed, complementing it with the study of the arguments' base score. The overall system architecture is illustrated in Fig.~\ref{fig:robot_decision_intro}.

    In summary, we contribute to the literature by:

    \begin{itemize}
        
        \item Presenting a novel framework for representing the multi-user preference conflicts, its properties, benefits and capabilities through the means of an assistive robot scenario example.
        \item Proposing a novel algorithm, based on QBAFs, for endowing robots with the capability of resolving multi-user preference conflicts, adaptable to the possibly changing context, and extended to the case where preferences can be positive or negative.
        
    \end{itemize}

\section{Preliminaries}

    The following section first briefly introduces the formal definitions and evolution of AFs, and then continues with the definitions, capabilities, and properties of the QBAFs, which is the framework used.

    \subsection{Introduction of Argumentation Frameworks}

        Abstract Argumentation Frameworks are a source for transparent practical reasoning that has been widely studied in the Artificial Intelligence field. They were first introduced and defined by Dung:
        
        \begin{defn}\textnormal{\cite{Dung_JLP95}} An Argumentation Framework is represented as a pair $\mathcal{AF} = \langle\mathcal{A,R^-\rangle}$ where $\mathcal{A}$ is the set of arguments, and $R^-$ is a binary relation on $\mathcal{A}$ and $R^-\subseteq\mathcal{A}\times\mathcal{A}$. The meaning of $\alpha R^-\beta$ is that the argument $\alpha$ attacks $\beta$.
        \end{defn}
        
        With the attack relation between arguments, the concept of defence appeared. For example, in Fig.~\ref{fig:example_attacks}, A attacks B and B attacks C, thus A is defending C since it attacks an argument that attacks C. With those definitions, the semantics for accepting arguments were introduced, being a set of arguments $S\subseteq\mathcal{A}$:

        \begin{itemize}

        \item \textit{Conflict-free:} $S$ is conflict-free iff there is no $B, C \in \mathcal{S}$ such that $B$ attacks $C$.
    
        \item \textit{Admissible:} $S$ is admissible iff it defends all of its elements.
    
        \item \textit{Complete:} $S$ is a complete extension iff it is an admissible set that contains any argument it defends.
    
        \item \textit{Preferred:} $S$ is a preferred extension iff it is a maximal admissible set.
    
        \item \textit{Stable:} $S$ is a stable extension iff it attacks any argument in $\mathcal{A}\setminus S$.
    
        \item \textit{Grounded:} $S$ is a grounded extension iff it is a minimal complete extension.
    
        \item \textit{Ideal:} $S$ is an ideal extension iff it is a maximal admissible set contained in every preferred extension.

        \end{itemize}
        
        One of the most relevant extensions of AFs is the addition of the support relation, since initially, only attacking arguments existed. Those frameworks are known as Bipolar Argumentation Frameworks (BAFs). 
    
        \begin{defn}\textnormal{\cite{Amgoud_IJIS08}}
            An Abstract Bipolar Argumentation Framework $\langle \mathcal{A}, R^-, R^+\rangle$ is a set of $\mathcal{A}$ arguments with binary relations $R^-$ for attacks and $R^+$ for supports.
        \end{defn}
        
        The supporting arguments brought new semantics and definitions of relations, such as direct/indirect attackers or supporters of an argument, the safety sets, which are the ones that do not attack and support the same argument, and new admissible and preferred extensions. In the example from Fig.~\ref{fig:example_supports}, the set \{A, C\} is not safe since the set is attacking and supporting B. 
    

        \begin{figure}[t]
            \begin{subfigure}{0.47\textwidth}
                \centering
                \includegraphics[width=1\linewidth]{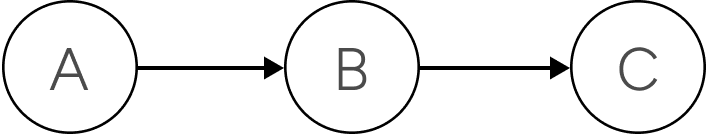}
            \caption{AF with A attacking B and B attacking C.}
            \label{fig:example_attacks}
            \end{subfigure}
            \hfill
            \begin{subfigure}{0.47\textwidth}
                \centering
                 \label{af:supp_att}
                 \includegraphics[width=1\linewidth]{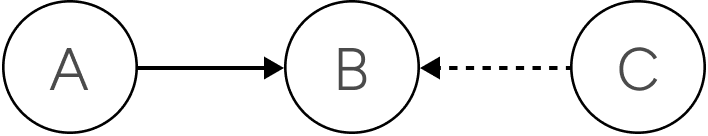}
            \caption{AF with A attacking B and C supporting B.}
             \label{fig:example_supports}
            \end{subfigure}
            \caption{Argumentation Framework Examples.}
            \label{fig:af_examples}
        \end{figure}

    \subsection{Quantitative Bipolar Argumentation Frameworks}
    \label{sec:preliminaries_qbaf}

        The literature in argumentation has been expanded into gradual argumentation, in which the arguments~\cite{Amgoud_ijcai17} or the relations~\cite{Dunne_AI11} are given a weight. The latest extensions in gradual argumentation are Quantitative Bipolar Argumentation Frameworks (QBAF). This framework will be used in this work. In QBAFs, the arguments are assigned an initial weight known as the \textit{base score}, which is used to compute the final strength of each argument.
        
        \begin{defn}\textnormal{\cite{Amgoud_IJAR18}} A QBAF is represented as a quadruple $\langle \mathcal{A}, \tau, \mathcal{R}^+, \mathcal{R}^-\rangle$, being $\tau: \mathcal{A} \rightarrow[0,1]$ the base score function of the arguments $\mathcal{A}$, the attack relations $\mathcal{R}^- \subseteq \mathcal{A}\times\mathcal{A}$ and the support relations $\mathcal{R}^+ \subseteq \mathcal{A}\times\mathcal{A}$ $(\mathcal{R}^- \cap \mathcal{R}^+ = \emptyset)$. If the tuple of arguments $(\beta,\alpha)$ is in $\mathcal{R}^- \ (\mathcal{R}^+)$, then the argument $\beta$ attacks (supports) $\alpha$.\end{defn}

        The strength of an argument represents its robustness against attacks and its acceptance. An example of a QBAF is represented in Fig.~\ref{fig:example_qbaf}. The arguments' final strengths are computed using a \textit{gradual semantics}. 
        
        \begin{defn}\textnormal{\cite{Baroni_IJAR19}} A gradual semantics $\sigma$ is a function that evaluates a QBAF by attributing values $\sigma(\alpha) \in [0,1]$ to every argument $\alpha \in \mathcal{A}$ as their strength.\end{defn}

        \begin{figure}
            \centering
            \includegraphics[width=0.25\linewidth]{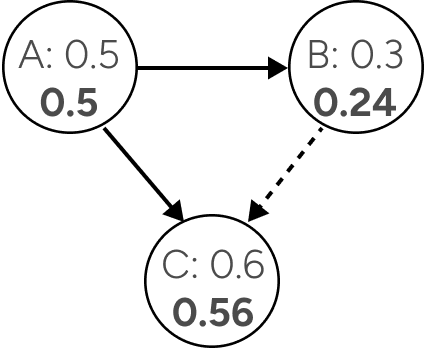}

            \caption{Example of a simple QBAF. The nodes contain the argument name (A, B, C), the base score on top, and the final strength in bold at the bottom. Continuous lines are attacks, and the dashed line is a support.}
            \label{fig:example_qbaf}
        \end{figure}
        
        Different gradual semantics exist in the literature, such as the \textit{DF-QuAD}~\cite{Rago_icpkrr16}, the \textit{Restricted Euler-Based}~\cite{Amgoud_IJAR18}, or the \textit{Quadratic Energy Model}~\cite{Potyka_pkrr18}. The approach used in this work is the Quadratic Energy Model since it satisfies all the properties of gradual evaluation methods~\cite{Amgoud_IJAR18, Potyka_sum19} and additionally, the final strength of the arguments tends to converge in fewer iterations than other approaches in cyclic graphs~\cite{Potyka_comma24}, and time is often crucial in some robotic applications. The converged gradual semantics using the Quadratic Energy Model are computed as: 
    
        \begin{equation}
            \sigma^{Q E}(\alpha)= 
            \begin{cases}
            \tau(\alpha)-\tau(\alpha) \cdot \frac{E_\alpha^2}{1+E_\alpha^2} & \text { if } E_\alpha \leq 0 \\ \tau(\alpha)+(1-\tau(\alpha)) \cdot \frac{E_\alpha^2}{1+E_\alpha^2} & \text { if } E_\alpha>0
            \end{cases}
            \label{eq:final_strength_QE}
        \end{equation}
        where $\alpha$ is any argument of the framework, and $E_\alpha$ is the aggregation strength of all attackers and supporters of $\alpha$, which is computed as:
    
        \begin{equation}
            E_\alpha=\sum_{\left\{\beta \in \mathcal{A} \mid(\beta, \alpha) \in \mathcal{R}^+\right\}} \sigma^{Q E}(\beta)-\sum_{\left\{\beta \in \mathcal{A} \mid(\beta, \alpha) \in \mathcal{R}^-\right\}} \sigma^{Q E}(\beta),
            \label{eq:aggregation_strength_QE}
        \end{equation}
        being $\beta$ the arguments with a direct edge towards $\alpha$. The final strengths typically range between 0 and 1; intuitively, the attackers move the final strength to 0 while the supporters move it to 1. 
    
        The final strength computation of acyclic QBAFs (acyclic graphs) with $\mathcal{R}$ relations ($\mathcal{R}^+ \cup \mathcal{R}^-$) and $A$ arguments converges in at least $n=\abs{A}$ iterations and in linear time $\mathcal{O}(m+n)$, where $m=\abs{R}$ and $n=\abs{A}$~\cite{Potyka_aamas19}. For cyclic QBAFs, it is less intuitive, and instead of computing the final strengths following the arguments' topological order towards the root arguments, the strengths are computed dynamically until convergence. 



        The apparition of these frameworks introduced new semantics and properties to the literature. Gradual evaluations (should) satisfy a set of properties that are well-defined in the literature~\cite{Amgoud_IJAR18, Baroni_IJAR19}. These properties dictate how arguments attack and support each other, as well as their influence on other arguments, in order to achieve a coherent evaluation of arguments. In other words, these properties include: 
        \begin{itemize}
            \item \textit{Anonymity}: Strength values should not depend on the identity of the argument. Two isomorphic QBAFs with the same base scores will have the same strengths.
            \item \textit{Independence}: Disconnected subgraphs should not affect each other.
            \item \textit{Directionality}: The final strength of an argument depends exclusively on its base score and its predecessors' strength.
            \item \textit{Equivalence}: Arguments with equal status should be equally evaluated.
            \item \textit{Stability}: If an argument has no attackers or supporters, its strength is its base score.
            \item \textit{Neutrality}: Arguments with 0 strength do not influence others.
            \item \textit{Monotony}: Adding attackers (supporters) only weakens (strengthens) an argument.
            \item \textit{Reinforcement}: Strengthening attackers (supporters) or weakening (strengthening) supporters only weakens (strengthens) an argument.
            \item \textit{Resilience}: Extreme values (0 and 1) cannot be taken unless the base score was already an extreme value.
            \item \textit{Franklin}: An attacker and a supporter with the same strength cancel each other.
            \item \textit{Weakening (Strengthening)}: An argument's strength must be smaller (bigger) than its base score when attackers (supporters) dominate the supporters (attackers).
        \end{itemize}

        Those and other properties from the literature were grouped into more general properties in~\cite{Baroni_IJAR19}, named \textit{balance} and \textit{monotonicity}. In a nutshell, balance expresses the idea that if attackers and supporters are equally strong, then the argument's final strength should be equal to the base score, and if attackers (supporters) are stronger, then the strength should not be more (less) than the base score. Monotonicity captures the notion that if you strengthen an argument's base score or its supporters, or weaken its attackers, the final strength should not decrease.
        
        Afterwards, a couple more properties have been added:
        \textit{Duality}~\cite{Potyka_pkrr18}, attacks and supports are treated equally; and \textit{Open-mindedness}~\cite{Potyka_sum19}, the strength of an argument can become close to 0 or 1 independently of its base score if there is enough magnitude of attackers or supporters. In particular, the Quadratic Energy Model, which is used in our framework, satisfies all those properties.

\section{Related work}

    Personalisation has become key for effective HRI, with extensive research demonstrating successful implementations through machine learning, adaptive control systems, and behavioural optimisation. However, these approaches have primarily addressed single-user scenarios, leaving multi-user preference conflicts systematically unaddressed. AFs emerge as a promising solution to bridge this gap, offering formal mechanisms for conflict resolution through structured reasoning. In the following, we review the main contributions to the state of the art.

    \textbf{Single-user preferences:} Traditional methods for adapting a robot's behaviour to a user's preferences or conditions encompass various disciplines within artificial intelligence. Central to this is the concept of user modelling, where the robot builds a representation of the user's characteristics to inform its actions. This model can be based on a variety of information, including demographic data~\cite{Maroto_EAAI23}, skill level in a training context~\cite{Schadenberg_CSR17}, or real-time behavioural cues observed during an interaction~\cite{Dell_EAAI22}. Among these methods, Reinforcement Learning (RL) is by far the most used approach in assistive robotics. For instance, Park~\etal\cite{Park_aaai19} used RL to maximise children's learning retention in long-term interactions with a robot teacher, where the reward function is designed to account for both children's engagement and their learning gains. 
    Similarly, Tsiakas~\etal\cite{Tsiakas_MDPI18} proposed an RL framework where the reward function combines explicit feedback, like task performance, with implicit feedback from the user's observed task engagement.
    However, designing a reward function that accounts for several heterogeneous parameters can be challenging. Therefore, other studies propose using Inverse Reinforcement Learning (IRL), which infers the underlying reward function from expert demonstrations rather than requiring it to be specified manually.
    Following this principle, Sugiyama~\etal\cite{Sugiyama_interspeech12} learn the appropriate reward function by modelling a user's pairwise decisions, effectively inferring preferences from their choices between two options.
    More recently, Andriella~\etal\cite{Andriella_UMUAI23} introduced the CARESSER framework, designed to actively learn personalised robot-assistive behaviour for cognitive training therapy. Here, the reward function was learnt by combining a therapist's expertise and their demonstrations, after which a unique policy was estimated using classical RL approaches.
    Beyond learning-based methods, classical AI planning has also been employed to achieve robot adaptiveness. This approach typically involves creating a symbolic representation of the user's needs, which is then integrated directly into the planning domain. A recent work by De Benedictis~\etal\cite{Benedictis_UMUAI23} developed a system for personalising cognitive stimulation for older adults. Here, the users' needs, capabilities, and preferences were directly represented in the planning knowledge base. Canal~\etal\cite{Canal_AR18} proposed a symbolic high-level task planning for assistive shoe dressing, in which the user's preferences were gathered from answering simple questions and integrated into the planning domain to generate appropriate assistance. A significant and common limitation of the aforementioned approaches is their foundation in a single-user paradigm. Whether it is a manually engineered reward function for RL, a set of expert demonstrations for IRL, or a coherent set of preferences for a planner, these methods are all designed to create a unified model tailored to an individual user. This paradigm is fundamentally challenged in multi-user scenarios where conflicting preferences make decisions more difficult, especially because preferences are often subjective for each user and challenging to quantify. 

    The introduction of Foundation Models and Large Language Models (LLMs) has allowed robots to better interpret user inputs and their surroundings. For instance, Wu~\etal\cite{Wu_AR23} proposed a method in which the user preferences are known and input into the prompt, and a robot in charge of household cleaning is capable of adapting its actions according to the user's preferences. Differently, in the work presented by Jenamani~\etal\cite{Jenamani_arxiv24}, a robot-assisted feeding system interprets verbal inputs to obtain the user's preferences, and also interprets the visual context, to generate a plan for feeding the user. Even though those works are based on single-user preferences, LLMs could be used to select between multiple users' preferences, but the approach would be subject to potential hallucinations or biases in their training and would lack transparent decisions if used end-to-end~\cite{Shen_EAAI25}.

    \textbf{Multi-users preferences:} Some research has addressed this issue by developing multi-agent, multi-objective systems that shape utility functions to resolve conflicts between agents with competing goals. For instance, Ren~\etal\cite{Ren_aps23} proposed the Binary Branching Multi-Objective Conflict-Based Search algorithm for finding the Pareto-optimal frontier. The work proposed by Huang~\etal\cite{Huang_ESA25} utilises a Multi-Agent Reinforcement Learning approach considering uncertainty for air traffic management, taking into account possible conflicting trajectories. They incorporate a reward-shaping mechanism to enrich the reward mechanism. Approaches using Game Theory have also been explored for solving this issue, such as the work from Bashir~\etal\cite{Bashir_IJIS21}, where they design a conflict resolution model considering all the parties and the payoff of different strategies and achieve more realistic and accurate results. These works, however, assume complete and quantitative information, lacking qualitative reasoning and intuitive explanations for their outcomes. Additionally, they depend on manually specified utility functions, which significantly limits their adaptability to diverse or evolving user preferences. There exist works in which preference conflicts are considered, for instance, the work from Bernatskiy~\etal\cite{Bernatskiy_alife14} adapts a robot behaviour according to multiple users' preferences, including possible conflict scenarios. They use evolutionary algorithms for learning the preferences and selecting the proper robot behaviour. Their objective is to optimise a process that requires multiple choices for a robot's behaviour. In case of preference conflict, they check the robot's performance for each of the preferences and select the most optimal one. Another work, which does not use robots, but includes multiple stakeholders with possibly competing preferences, is the one from Shrivastava~\etal\cite{Shrivastava_EAAI25}, in which a recommender system uses a deep learning method to learn the preferences, and then a multi-objective approach optimises the recommendations based on some objective parameters, such as consumption, retention, and exposure. Both approaches use objective data to make decisions or recommendations, since their task is clear in terms of whether it is being properly performed or optimised, unlike in some HRI preference conflicts.

    \textbf{Computational argumentation in robotics:} The application of CA techniques in robotics is an emerging but promising area. While argumentation has seen extensive use in adjacent fields like conversational AI and chatbots for managing dialogue and explaining decisions~\cite{Castagna_JAIR24}, its adoption in HRI is less developed. In the robotics literature, a handful of studies have begun to explore this potential. Wang~\etal\cite{Wang_mfi14}, for example, used an AF to allow a service robot to reason over a single user's preferences and explain its subsequent actions. In a study presented by Azhar~\etal\cite{Azhar_roman16}, an argumentation-based dialogue system was developed that enabled a human and a robot to collaboratively reach a decision. Although these studies demonstrate the value of argumentation for human-robot decision-making, they are fundamentally limited to single-user scenarios and do not provide a mechanism for resolving conflicts between multiple human stakeholders.     
    To the best of our knowledge, the work presented by Ferretti~\etal\cite{Ferretti_AI17} is the first to use dynamic features of the environment in AFs for robot decision-making. A robot that cleans a store uses a dynamic AF created from the scenario configuration and some rules to make decisions and finish the task. The work presented by Ayoobi~\etal\cite{Ayoobi_TASE21} proposes that online incremental argumentation can be used for the robot to handle recoveries from failures originating from changes in the environment, proving to be generalised with few observations in large state spaces, contrary to RL and contextual bandits approaches.

    \textbf{Research Gap:} The presented literature is limited by the following: firstly, it mostly involves a single user and their preferences for developing a task, while in our work, multiple users' preferences are considered. Secondly, prior research typically assumes objective scenarios with measurable optimal outcomes. However, in HRI, optimality is inherently subjective and can vary across users. Our approach explicitly accounts for these subjective preferences in the robot's decision process. Thirdly, current HRI systems that address conflict resolution lack mechanisms for generating transparent decisions, while our framework provides transparency by design. Finally, although some studies incorporate dynamic arguments based on the robot or task status, they overlook the possibility that the arguments and preferences of the users might change over time, a key challenge in HRI that our framework effectively addresses.

\section{Multiple users' preference selection system}
    This section introduces the modules (see  Fig.~\ref{fig:robot_decision_intro}) of our framework used to select a preference in a scenario with multiple users with possibly competing preferences for a specific action. 
    Firstly, the robot observations and their history generate the task and users' arguments in the arguments module (see Sec.~\ref{sec:robot_observations_and_arguments_module}). Next, these arguments are fed into the preference selector module formalised as a QBAF to determine the selected preference (see Sec.~\ref{sec:preference_selector}). Finally, the selected preference is input to the action sequence executor module to parameterise the action. It is important to note that developing a decision-making system to select the best action in a given state is outside the scope of this work.
    
    For a clearer understanding, in the following section, we define several key terms: \textit{Features} refer to the attributes that the robot can choose from, \textit{Options} are the possible values that each feature can assume, and \textit{Preferences} indicate the users' preferred options for each feature. For example, a feature of the robot could be its velocity, with options such as fast, slow, or very slow. One user may prefer the robot to move quickly, while another might prefer a slower pace.

    \subsection{Robot observations and arguments modules}
    \label{sec:robot_observations_and_arguments_module}

        In an HRI setting, it is crucial to know the user’s preferences and their reasons, as well as gather contextual information to generate adaptive robot behaviour, an expected characteristic in robotic domains. For this reason, this module is in charge of obtaining the users' preferences and arguments, and also of observing the context or some indicators of the users or task status. Those observations, and their changes over time, are sent to the arguments module, becoming arguments, since they can be meaningful for the decision. Overall, there are three types of arguments:

        \begin{itemize}
             \item \textit{Users Arguments}: Arguments provided by the users. These arguments reflect their preferences over the available options. An example of a user argument is: \textit{``I do not want to repeat the test because I do not feel safe performing it''}.
             \item \textit{Task Arguments}: Arguments whose origin is based on robot observations. This type of argument is a key and distinguished characteristic when applying argumentation to HRI, since in this context, the environment and the user's state are dynamic and can impact the robot's observations. They can be defined by a user or through a consensus among multiple experts, for example, in a participatory design process. An example of a task argument is: \textit{``Risk of falling detected during the test execution''}.
             \item \textit{Option Arguments}: Arguments that represent the different options for a given feature. They are the possible outputs of the preference selector module. To ensure fairness, these arguments are assigned equal base scores. A pair of \textit{Option Arguments} can be: \textit{``Repeat test''} and \textit{``Not repeat test''}.
         \end{itemize}

    \subsection{Preference selector module}
    \label{sec:preference_selector}
        This module serves as the core computational engine for resolving conflicts and determining the robot's course of action. Upon receiving the users' preferences and arguments, the module initiates a four-stage process. Firstly, it analyses the preferences to detect any conflicts (see Sec.~\ref{sec:preferences_types_conflicts}). Second, it formalises the conflict using our proposed AF (see Sec.~\ref{sec:mup_qbaf_framework}). Third, a dedicated algorithm processes this framework to select a winning preference and decide the parameters of the robot's action (see Sec.~\ref{sec:preference_selector_algorithm}). Finally, we demonstrate that this decision-making process adheres to several properties that ensure its responsiveness in dynamic contexts (see Sec.~\ref{sec:dynamic_decision_properties}).

    \subsubsection{Preferences and conflict types}
    \label{sec:preferences_types_conflicts}

        There are three possible categorical preferences a user can have for a given option. They can be positive (preferred, represented as $+$), negative (undesirable, represented as $-$), or indifferent (without any preference, represented as $0$)~\cite{Bistarelli_soft05}.  
        
        We consider a context with multiple \textit{Option Arguments} $O$, multiple users $U$, and where the preferences are defined as $p: U  \to  2^{O \times \{ +, -, 0 \}}$, such that if $(o,s),(o,s') \in p(i) $ then $s = s'$. We extend $p(i)$ for all $i \in U$ by defining $\hat{p}$ in the following way:

        \begin{enumerate} 
        \item $(o,s) \in \hat{p}(i)$ if $(o,s) \in p(i)$
        \item $(o,0) \in \hat{p}(i)$ if $(o,\_) \notin p(i)$
        \item nothing else is in $\hat{p}(i)$.
        \end{enumerate}



        With this formulation of the preferences, the scenarios without conflicting preferences (NC) are the following:

        \begin{enumerate}[label=(NC.\arabic*), leftmargin=*, align=left, labelwidth=!, itemindent=0pt, widest=9]

            \item All users have the same preferences:
                \begin{equation*}
                \label{eq:no_conflict_1}
                    \forall \ i,j \in U. \ \hat{p}(i) = \hat{p}(j) 
                \end{equation*}
            \item There is one option preferred by some, and the rest are (1) unopposed to it and (2) do not show a preference for any option:
                \begin{equation*}
                \label{eq:no_conflict_2}
                    \exists o \in O. \ \exists U' \subseteq U. \ |U'| \geq 1 \land \forall i \in U'. \ (o, +)\in \hat{p}(i) \land  \forall o' \in O .  \ o \neq o' \to 
                \end{equation*}
                \begin{equation*}
                    (o',+) \notin \hat{p}(i) \land \big(\forall j \in U \setminus U'. \ (o, 0) \in \hat{p}(j) \land (o', +) \notin \hat{p}(j) \big)
                \end{equation*}

            \item There is no option preferred by any user, but there is at least one option that is indifferent for all:

            \begin{equation*}
            \label{eq:no_conflict_3}
                 \forall i \in U .\ \big(\forall o \in O.\ (o,+)\notin P_i \land  \exists o'\in O.\ (o',0) \in \hat{p}(i)\big).
            \end{equation*}

        \end{enumerate}

        When none of the previous conditions are satisfied, the scenarios with conflicting preferences (C) are:

        \begin{enumerate}[label=(C.\arabic*), leftmargin=*, align=left, labelwidth=!, itemindent=0pt, widest=9]

        \item There is an option that is preferred and disliked by different users: 
        \begin{equation*}
        \label{eq:conflict_1}
             \exists o \in O.\ \exists \ i,j \in U. \ i \neq j \land (o,+)\in \hat{p}(i) \land (o,-) \in \hat{p}(j) 
        \end{equation*}
        
        \item There are at least two agents where one prefers an option that is not preferred by the other one:
        \begin{equation*}
        \begin{split}
        \label{eq:conflict_2}
            \exists \ i,j \in U. \  \exists \ o,o'\in O.\  o' \neq o \ \land \ i \neq j \ \land \ (o,+) \in \hat{p}(i) \ \land \ (o',+) \in \hat{p}(j) \\
            \land \  (o,+) \notin \hat{p}(j) \ \land \ (o',+) \notin \hat{p}(i)
        \end{split}
        \end{equation*}

        \item There are at least two agents where one rejects an option that is not rejected by the other one:
        \begin{equation*}
            \begin{split}
            \label{eq:conflict_3}
                \exists \ i,j \in U. \  \exists \ o,o'\in O.\  o' \neq o \ \land \ i \neq j \ \land \ (o,-) \in \hat{p}(i) \ \land \ (o',-) \in \hat{p}(j) \\
                \land \  (o,-) \notin \hat{p}(j) \ \land \ (o',-) \notin \hat{p}(i).
            \end{split}
        \end{equation*}
        \end{enumerate}

        Since this work uses QBAFs to resolve preference conflicts, a user favouring an option will naturally provide arguments supporting it, while likely attacking arguments for the alternatives. Conversely, a user who finds an option undesirable will attack its supporting arguments and likely provide support for others.

    \subsubsection{The MUP-QBAF framework}
    \label{sec:mup_qbaf_framework}

        The arguments received from the arguments module, and defined in the previous section (see Sec.~\ref{sec:robot_observations_and_arguments_module}), are used to construct a Multi-User Preferences Quantitative Bipolar Argumentation Framework (MUP-QBAF):

        \begin{defn} 
            A MUP-QBAF framework is a 7-tuple $\langle \mathcal{A}, \tau, \mathcal{R}^+, \mathcal{R}^-,O, U,  \hat{p} \rangle $ such that:
            \begin{itemize}
                \item $\mathcal{A}$ is a finite set of arguments;
                \item $\tau: \mathcal{A} \to [0,1]$ is the base score function for the arguments;
                \item $\mathcal{R}^-\subseteq \mathcal{A} \times \mathcal{A} \ (\mathcal{R}^+)$ are the attack (support) relations;
                \item $O \subseteq \mathcal{A}$ are arguments representing options; 
                \item $U$ is the set of users;
                \item $\hat{p}: U \to  2^{O \times \{ +, -, 0 \}}$ are the users' preferences for the options $O$;
            \end{itemize}
        \end{defn}
        
        The structure of a MUP-QBAF is considered a set of trees with as many roots as available options. That is defined as follows~\cite{Rago_icpkrr23}:

        \begin{defn}
            Let $\mathcal{M}$ be a MUP-QBAF $\langle \mathcal{A}, \tau, \mathcal{R}^+, \mathcal{R}^-,O, U,  \hat{p} \rangle$. For any arguments $a,b\in \mathcal{A}$, let a path from $a$ to $b$ be defined as a sequence of relations $(c_0,c_1),...,(c_{n-1},c_n)$ of length $n>0$ where $c_0=a$ and $c_n=b$, and, for any $1\leq i \leq n, (c_{i-1}, c_i)\in \mathcal{R}^+ \cup \mathcal{R}^-$. Then, given a set of \textit{Option Arguments} $\subseteq\mathcal{A}$, $\mathcal{M}$ is a MUP-QBAF for $O$ iff i) $\nexists a\in \mathcal{A}\setminus \{ O\}$ such that $\exists o \in O$ where $(o, a)\in \mathcal{R}^+\cup\mathcal{R}^-$ ii) $\forall a\in \mathcal{A} \setminus \{O \}$ there is a path from $a$ to at least one $o \in O$; and iii) $\nexists a \in \mathcal{A}$ with a path from $a$ to $a$.
            \label{defn:structure_mup_qbaf}
        \end{defn}

        With these definitions established, it is now possible to define an algorithm to select between multiple users' conflicting preferences.
    
    \subsubsection{Multi-user preference selector algorithm}
    \label{sec:preference_selector_algorithm}

        Once the framework is initialised, we propose Alg.~\ref{algorithm:decision_making} to select between the different options for that specific action. 
        The algorithm takes as input a MUP-QBAF framework $\mathcal{M}$ and computes the gradual semantics $\sigma$ of its arguments with the evaluation method $\nu$. 


        \begin{algorithm}[tbp!]
            \KwIn{
                 $\mathcal{M}=\langle \mathcal{A}, \tau, \mathcal{R}^+, \mathcal{R}^-,O, U,  \hat{p} \rangle $,
                 $\nu:\{'Quadratic \ Energy', 'DF-QuAD'\}$,
                 $g:2^O \times \mathcal{M} \times \nu \to O$
                }
            \KwOut{$o\in O$}

                $ \sigma  \xleftarrow{} compute\_gradual\_semantics(\nu, \mathcal{M})$\;
                \If {C.1 or C.2 or C.3}{
                    $o^* \xleftarrow{}  g (\{o | \sigma(o) = \max_{o' \in O} \sigma(o')\}, \mathcal{M}, \nu)$\;
                    \Return $o^*$
                }
                \Else{
                    $O^+ \leftarrow\{o |\exists i\in U.(o,+) \in \hat{p}(i)\}$\;
                    $O^0 \leftarrow\{o |\exists i\in U.(o,0) \in \hat{p}(i)\}$\;
                    $O^- \leftarrow\{o |\exists i\in U.(o,-) \in \hat{p}(i)\}$\;
                    \Switch{No Conflict}{
                        \Case{NC.1}{ 
                            \uIf{$|O^+| \geq 1$}{
                                \Return $g(\{o | o\in O^+ \land \sigma(o) = \max_{o' \in O^+} \sigma(o')\}, \mathcal{M}, \nu)$
                            }
                            \ElseIf{$|O^0| \geq 1$}{
                                \Return $g(\{o | o\in O^0 \land \sigma(o) = \max_{o' \in O^0} \sigma(o')\}, \mathcal{M}, \nu)$
                            }
                            \uElse{
                                \Return $g(\{o | o\in O^- \land \sigma(o) = \max_{o' \in O^-} \sigma(o')\}, \mathcal{M}, \nu)$
                            }
                        }
                        \Case{NC.2}{
                            \Return $g(\{o | o\in O^+ \land \sigma(o) = \max_{o' \in O^+} \sigma(o')\, \mathcal{M}, \nu)$
                        }
                        \Case{NC.3}{
                        \tcc{The $O^+$ set is empty}
                            \Return $g(\{o | o\in O^0\setminus O^- \land \sigma(o) = \max_{o' \in O^0\setminus O^-} \sigma(o')\   , \mathcal{M}, \nu)$
                        }
                       
                    }
                }
            \caption{Multi-User Preferences Conflict Resolution (MUPCR)}
            \label{algorithm:decision_making}
        \end{algorithm}

        The first step of the algorithm is to calculate the final strengths of the \textit{MUP-QBAF's} arguments (line 1). The next step consists of checking if any of the conditions for conflicting scenarios from Sec.~\ref{sec:preferences_types_conflicts} are satisfied (line 2). In case of conflict, the algorithm will select an option $o^*\in O'$ by running $g$ (line 3), which is a function that selects among the set of options $O' \subseteq O$ with the highest strength: $o^* = g \{o | \sigma(o) = \max_{o' \in O} \sigma(o'), \mathcal{M}, \nu\}$. Those options with the maximum strength are input into the $g$ function, which selects an option given some criteria. If there is only one option in $g$, the conflict is considered resolved, and the selected option is that one. 

        \begin{property}
            \textnormal{(Conflict Resolved)} Given a MUP-QBAF $\mathcal{M}$ and a gradual semantics $\sigma$, the conflict between the preferences is resolved iff $\ \exists o, o' \in O. \forall o'\neq o \to \sigma(o) > \sigma(o')$.
        \end{property}

        In contrast, if the function $g$ has more than one option as input, the algorithm detects a tie and tries other arbitrary methods to select an option. That arbitrary method has to be decided within the context, e.g., the robot could ask for user feedback to update the argumentation framework to break the tie (see Alg.~\ref{algorithm:g_update_framework}), or could return the most efficient option from the available options.

        \begin{algorithm}[t]
            \KwIn{
                $O' \subseteq  O,\mathcal{M}=\langle \mathcal{A}, \tau, \mathcal{R}^+, \mathcal{R}^-,O, U,  \hat{p} \rangle, \nu $
                }
            \KwOut{$o^*\in O'$}

                \If{$|O'|=1$}{
                    \Return $o\in O'$}
                \Else{
                    \tcc{Ask users for more arguments to break the tie}
                    $\mathcal{M}'$ $\xleftarrow{}$ $update\_framework(\mathcal{M})$\;
                    \Return $MUPCR(\mathcal{M}', g_1, \nu)$
                }
            \caption{Example of $g$ - Robot Conflict Resolution which Proactively Updates Argumentation Framework in Ties}
            \label{algorithm:g_update_framework}
        \end{algorithm}

    \subsubsection{Properties of MUP-QBAFs dynamic adaptation}
        \label{sec:dynamic_decision_properties}

        A key challenge in HRI is to enable robots to immediately adapt to dynamic environments and user preferences. Traditional decision-making approaches often rely on static models or require retraining, limiting their responsiveness. Our work introduces a novel application of QBAFs to HRI, uniquely addressing scenarios in which: (i) observations are dynamic and context-dependent, robot actions (e.g., administering a test) directly influence the environment and user state, generating new arguments in real time; (ii) multi-stakeholder preferences must be reconciled, the robot balances conflicting inputs from users (e.g., caregivers, care recipients) while accounting for task-critical observations (e.g., fall risks).
        
        We will provide evidence that introducing new arguments into the MUP-QBAF alters the difference in the final strengths of the \textit{Option Arguments}. To achieve this, we will expand upon the definitions of pros and cons of an option argument within our framework, based on Def.~\ref{defn:structure_mup_qbaf}~\cite{Rago_icpkrr23}.
    
            \begin{defn}
                Let $\mathcal{M}$ = $\langle \mathcal{A}, \tau, \mathcal{R}^+, \mathcal{R}^-,O, \mathcal{U}, \hat{p} \rangle $ be any MUP-QBAF. Then, the con arguments and pro arguments for an option $o\in O$ are:
                \begin{itemize}
                    \item $pro(o) = \{ a \in \mathcal{A} \mid \exists p \in paths(a, o), where \ \abs{ p \cap \mathcal{R}^-}\ is \ even\}.$
                    
                    \item $con(o) = \{ a \in \mathcal{A} \mid \exists p \in paths(a, o), where \ \abs{ p \cap \mathcal{R}^-} \ is \ odd \}.$
                \end{itemize}
            \end{defn}
    
        By definition, an argument can present both pros and cons regarding the same \textit{Option Argument}. This situation can occur when there are multiple paths leading from the argument to the \textit{Option Argument}. However, the nature of MUP-QBAF suggests the involvement of more than one user, each with a preference that the algorithm uses to make a selection. As a result, the arguments they present will support their preferred options (pros) or oppose the unpreferred ones (cons).
        
        \begin{defn}
            For any MUP-QBAF, let $i\in U$ be a user, and $\mathcal{A}_i \subseteq \mathcal{A}$ be that \textnormal{user's arguments}, given an Option Argument o, a \textnormal{consistent user} is defined as:
            \begin{itemize}
                \item if $(o,+) \in \hat{p}(i): \exists\alpha \in \mathcal{A}_i.\alpha\in pro(o)$
                
                \item if $(o,-)\in \hat{p}(i): \exists\alpha \in \mathcal{A}_i.\alpha\in con(o)$
                \item if $(o,0) \in \hat{p}(i): 
                \big(\exists\alpha \in \mathcal{A}_i.\alpha\in con(o) \land \ \exists \beta \in \mathcal{A}_i. \beta \in pro(o) \big) \lor \big(\nexists\alpha \in \mathcal{A}_i.\alpha\in con(o) \land \ \nexists \beta \in \mathcal{A}_i. \beta \in pro(o)\big)$ 
                
            \end{itemize}
        \end{defn}
        The pros and cons arguments influence the strength differences between the \textit{Option Arguments} in a MUP-QBAF. By introducing a sufficient number of pros or cons arguments, the algorithm may alter its selection. The following property illustrates how pro and con arguments affect the strength differences of the \textit{Option Arguments}.

        \begin{property}\label{property:argument_attribution_discrimination}
           \textnormal{(Argument Addition Option Discrimination)}
           A gradual semantics $\sigma$ satisfies \textnormal{Argument Addition Option Discrimination} iff for any MUP-QBAF $\mathcal{M}$, a modified version of it $\mathcal{M}'$ with an argument $\alpha$  such that $\mathcal{A'}=\mathcal{A} \ \cup \alpha,$ $\tau'(\alpha)>0,$ $\tau'(\beta) = \tau(\beta)$ for all $\beta \in \mathcal{A}$, and a given options $(o_1,o_2) \in O$ such that $\sigma(o_1), \sigma(o_2) \in (0,1)$: If $\alpha \in pro(o_1) \lor \alpha \in con(o_2)$, and $\exists \gamma \in \mathcal{A}.(\alpha,\gamma)\in \mathcal{R}^+ \lor (\alpha,\gamma)\in \mathcal{R}^-.\sigma(\gamma)\in(0,1)$ then $\sigma'(o_1)-\sigma'(o_2)>\sigma(o_1) - \sigma(o_2)$.

        \end{property}

        \begin{proposition}
            The Quadratic Energy, the Euler-based, and the DF-QuAD models satisfy \textnormal{Argument Addition Option Discrimination}.
        \end{proposition}

        \begin{proof}
        \label{proof:adding_arguments}
            Following Proposition 2 from~\cite{Rago_icpkrr23}, and since the DF-QuAD, the Euler-based, and Quadratic Energy evaluation methods satisfy the monotonicity property, in the case with $\alpha \in pro(o_1)$, if $\alpha$ directly supports $o_1$, it will increase $o_1$'s final strength, if it is attacking an $o_1$ attacker, it will decrease that attacker's final strength (considering that attacker's final strength was not already 1), impacting positively on $o_1$'s final strength, and if it is supporting an $o_1$'s defender, it will also increase that defender's final strength (considering that defender's final strength wasn't already 1) increasing then $o_1$'s final strength. Given all those, then $\sigma'(o_1) > \sigma(o_1)$, and since $\sigma'(o_2) \leq \sigma(o_2)$, then $\sigma'(o_1)-\sigma'(o_2)>\sigma(o_1) - \sigma(o_2)$. In the case where $\alpha \in con(o_2)$, then $\sigma'(o_2) < \sigma(o_2)$, and since $\sigma'(o_1) \geq \sigma(o_1)$ then $\sigma'(o_1)-\sigma'(o_2)>\sigma(o_1) - \sigma(o_2)$. In the case where $\alpha \in pro(o_1) \land\alpha \in con(o_2)$, by extension of the previous cases, this also achieves $\sigma'(o_1)-\sigma'(o_2)>\sigma(o_1) - \sigma(o_2)$. The property only holds in the cases where $\sigma(o_1), \sigma(o_2) \in (0,1)$, for instance, if $\sigma(o_1)=1$ and $\alpha$ is pro($o_1$), then $\sigma'(o_1)=\sigma(o_1)=1$, and the property is not satisfied. 

        \end{proof}

        It is important to note that this property can also be applied to the process of deleting an argument. Given this property, the robot may change its selected option based on new arguments introduced by users or through the activation of task arguments. A change in selection will occur when the difference between options changes sign, e.g., $\sigma(o_1)-\sigma(o_2) > 0$ and $o_1$ is selected, and after adding new arguments, $\sigma'(o_1)-\sigma'(o_2) < 0$.
        
        Second, we prove that the difference between different \textit{Option Arguments} can change by modifying some arguments' base score. 

        \begin{property}\label{property:base_score_discrimination}
            \textnormal{(Base Score Option Discrimination)}
            A gradual semantics $\sigma$ satisfies \textnormal{Base Score Option Discrimination} iff for any MUP-QBAF $\mathcal{M}$, for any $\mathcal{M}'$ with $\mathcal{A'}=\mathcal{A}, \ \mathcal{R}^{'+}=\mathcal{R}^{+}, \ \mathcal{R}^{'-}=\mathcal{R}^{-} , $ for $\alpha \in \mathcal{A}, \ \tau'(\alpha) > \tau(\alpha)$ while $\forall \beta \in \mathcal{A}\setminus \{\alpha\}.\tau'(\beta)=\tau(\beta), \ $, and given options $o_1,o_2 \in O$, $\sigma(o_1), \sigma(o_2) \in (0,1)$: If $\alpha \in pro(o_1) \lor \alpha \in con(o_2),$  then $\sigma'(o_1)-\sigma'(o_2)>\sigma(o_1) - \sigma(o_2)$.

        \end{property}

        Note that this property can also extend to lowering the base score of an argument.  

        \begin{proposition}
            The Quadratic Energy, the Euler-based, and the DF-QuAD models satisfy \textnormal{Base Score Option Discrimination}.
        \end{proposition}

        The previous proof is also applicable to this property when gradual semantics satisfy the monotonicity property of QBAFs~\cite{Baroni_IJAR19}. These modifications may be given from user feedback or the robot's own observations.

        The presented system properties are focused on the possibility of changing the framework's output. In long-term interactions, the agents' preferences and arguments might change, and the system must allow and learn those changes. As it has been shown in the previous examples, the proposed system allows \textit{Non-stationary preferences and arguments}: 

        \begin{defn}
            A MUP-QBAF framework allows non-stationary preferences and arguments. Therefore, the agents' preferences and arguments, and consequently the decisions, can change over time.
        \end{defn}

        A relevant observed benefit of using a framework based on QBAFs in HRI contexts is their adaptability to changes, in a single interaction, without retraining to adapt the model to the new context. 

\section{Using MUP-QBAF in robotic frailty assessment}
\label{sec:robotic_frailty_assessment}

    This section presents an illustrative use case demonstrating the practical application and advantages of the proposed MUP-QBAF framework in an HRI scenario. The objective is to highlight the system’s ability to resolve conflicts between users with competing preferences, specifically in the context of a robot performing frailty assessments in older adults. 

    Firstly, the use case, the role of the robot, and the possible conflicts are explained (see Sec.~\ref{sec:clinical_scenario}). Then, an example of a possible AF (see Sec.~\ref{sec:constructing_af}) is provided. Moreover, some examples related to adaptation to dynamic context, which is crucial in HRI, are presented (see Sec.~\ref{sec:dynamic_decision_examples}). Finally, an analysis of the proposed AF in different scenarios, and the sensitivity of the arguments' base score is shown (see Sec.~\ref{sec:framework_evaluation_and_sensitivity_analysis}).

    \subsection{Clinical scenario and motivation}
    \label{sec:clinical_scenario}


    Frailty assessments are standard clinical procedures used to evaluate the physical, cognitive, mental, and functional conditions of older adults. These evaluations help healthcare professionals anticipate adverse outcomes, particularly in response to stressors such as hospitalisation or injury~\cite{Fried_TJG01}. The physical component of frailty is typically measured using standardised tests such as the Short Physical Performance Battery (SPPB)~\cite{Guralnik_JG94} and the Timed Up and Go (TUG)~\cite{Podsiadlo_JAGS91}. The functional domain is assessed through responses to a questionnaire known as the Patient Reported Outcome Measurements (PROM)~\cite{Broderick_Egems13}.

    Assistive robots have the potential and capabilities to conduct this assessment autonomously~\cite{Civit_hri24, Civit_roman24}. They can accompany a patient to an evaluation room, guide them through the different tests, and then provide accurate assessments along with additional relevant frailty-related metrics~\cite{Boumans_IJSR20, Calabrese_IJSR25}. 
    
    Conflicts between the care recipient and the caregiver can arise during the testing process. For example, if a patient performs a test incorrectly or an unexpected event occurs, the caregiver may prefer to repeat the test to ensure the validity of the measurement. In contrast, the care recipient might prefer not to repeat it, citing fatigue, discomfort, a safety issue, or time constraints. This sets up a multi-user preference conflict, where both parties have legitimate but opposing viewpoints. Specifically, the robot is requested to select between two options (\textit{Option Arguments}): repeat ($r$) the test that went wrong or not repeat it ($\neg r$).

    \subsection{Constructing the Argumentation Framework}
    \label{sec:constructing_af}

    \begin{figure*}[t]
        \centering
        \includegraphics[width=1\linewidth]{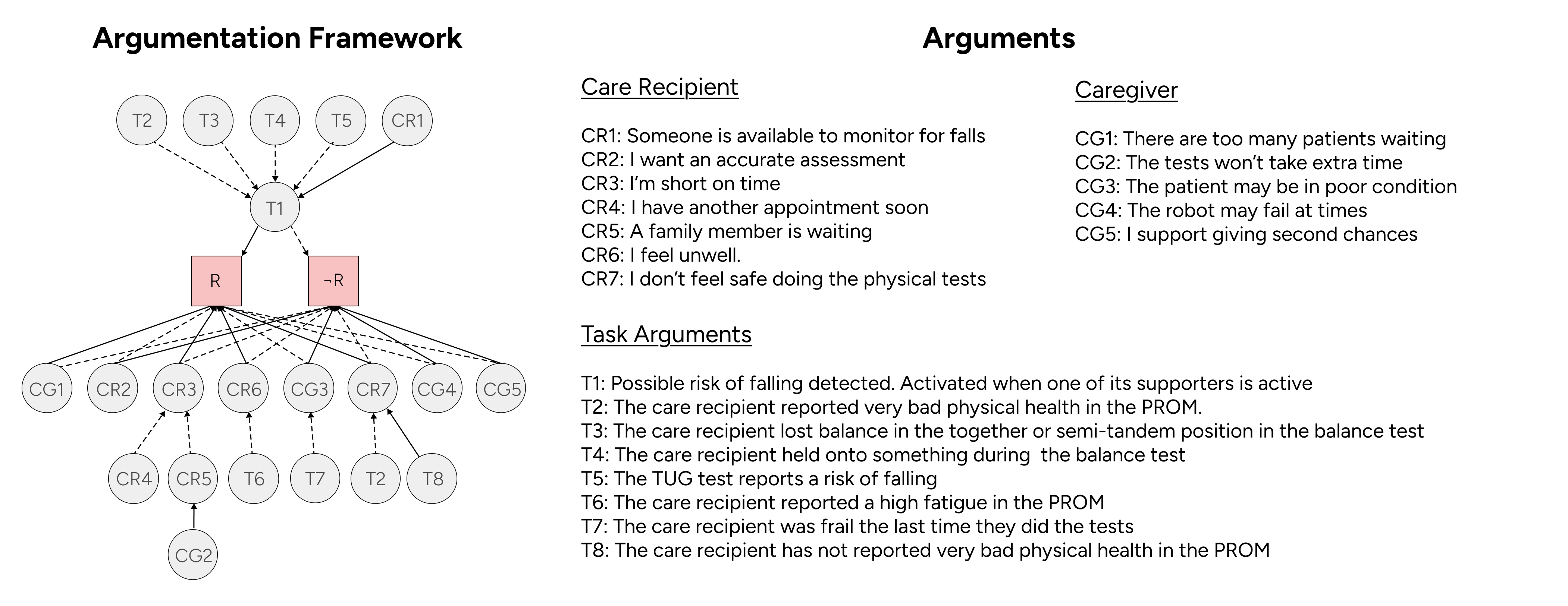}
        \caption{An example of the initial AF for determining whether to repeat a test in frailty assessments is shown on the left, along with the associated arguments on the right. The squared arguments highlighted in red represent the \textit{Option Arguments} (to repeat or not to repeat the test). The solid lines indicate attacks on the arguments, while the dashed lines represent supports. In this context, $T$ denotes task arguments, $CR$ refers to care recipient arguments, and $CG$ signifies caregiver arguments. The numbers following each argument serve to enumerate them. The arguments positioned above the \textit{Option Arguments} belong to the risk of falling observations, while those below are associated with the users.}
        \label{fig:framework_frailty}
    \end{figure*}

    We propose an AF that contains arguments gathered during co-design sessions with healthcare professionals and care recipients from a healthcare facility, along with in situ naturalistic observations during experimentation. The reason for this AF is that, in this particular case, there will be only a single session during which the robot will administer the tests to the patient, limiting the possibilities of learning or understanding the arguments that users can have over time. Before starting the session, both the caregiver and care recipient will select the arguments they agree on. The AF includes the users' arguments and relations in the different possible scenarios in which conflicts can arise, where the caregiver prefers to repeat a test, and the care recipient does not want to repeat it, or vice versa. Depending on the preferences and their reasons, some arguments will be activated and others deactivated. The task arguments are obtained through the PROM questionnaire and the user's performance in the tests, and from the performance in the previous tests. All these arguments together provide information about the user's risk of falling, which is an important factor to consider when deciding whether to repeat a test or not. The AF is depicted in Fig.~\ref{fig:framework_frailty}. 

    The conflicting scenarios are: (i) the caregiver prefers not to repeat the test and the care recipient prefers to repeat it, which includes the arguments $CG1$, $CR1$, $CR2$ and all the task arguments $TX$, and (ii) the caregiver prefers to repeat the test and the care recipient prefers not to repeat it, which includes the rest of the users arguments and the task arguments.

    In the context of frailty assessments, the robot is requested to measure the care recipient's performance in all the tests. During the execution of the tests, the robot makes observations that can be interpreted as arguments. For instance, if a potential risk of falling is detected, a task argument representing that risk is activated (T1), which attacks the repeat option argument and supports the no repeat option. As illustrated in Fig.~\ref{fig:framework_frailty}, the task argument (T1) is activated and can be strengthened by other arguments based on observations. These observations include: (T2) the patient reported very poor physical health in the PROM, (T3) the care recipient lost their balance during the SPPB test in both the together and semi-tandem positions, (T4) the care recipient needed to hold onto something for support during the balance test, and (T5) the performance on the TUG test indicated a risk of falling. Observe that these arguments are obtained incrementally from observations.

    \subsection{Dynamic adaptation examples}
    \label{sec:dynamic_decision_examples}

        A crucial requirement in HRI is their rapid adaptation to dynamic changes in the environment and user preferences. Here, we provide two examples illustrating the framework's adaptation in response to certain changes in the AF. 
        
        The first example involves the addition of new arguments to the AF.
        Indeed, in some scenarios, it is possible that the users want to provide additional arguments to the robots to express or reinforce their preferences, or it is also possible that robots make meaningful observations that should be considered in the preference selection process in the form of new arguments.

        \begin{figure*}[tbp!]
            \centering
            \includegraphics[width=1\linewidth]{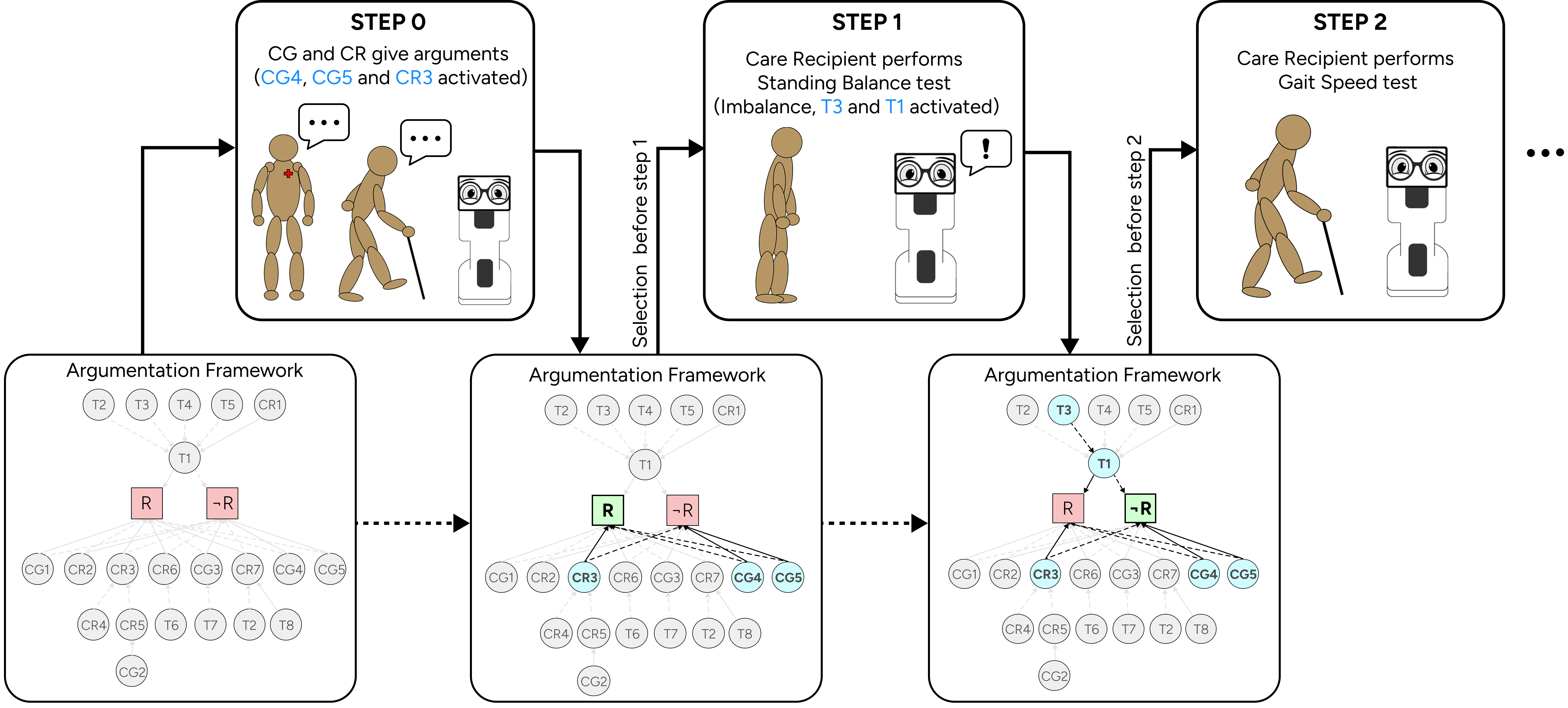}
            \caption{Representation of Example~\ref{example:frailty_adding_arguments}. The caregiver and care recipient give their arguments to the robot, and those are activated (in blue and bold) in the Argumentation Framework. R and $\neg$R are the \textit{Option Arguments} for repeating or not a test. Initially, the robot decides to repeat the tests, indicated by a rectangular argument in green. However, after activating arguments T1 and T3 due to an imbalance detected in the standing balance test, the robot changes its decision and decides not to repeat the tests.}
            \label{fig:example_frailty}
        \end{figure*}

        The example below shows the robot adaptation in such cases (satisfying Property~\ref{property:argument_attribution_discrimination}):

        \begin{example}
            \label{example:frailty_adding_arguments}
            The caregiver believes that the tests should be repeated because the robot sometimes fails to measure them accurately. They want to provide an opportunity for retaking the tests in case of any failures, which often occur when the user does not fully understand how to perform the tests correctly at first (arguments CG4 and CG5 are activated). On the other hand, the care recipient prefers not to repeat the tests due to time constraints (CR3 activated). Initially, all arguments have a base score of 0.5.
            In this scenario, CG4 and CG5 attack \textit{no repeat} and support \textit{repeat}, and CR3 does the opposite. The final strength for \textit{repeat} is $\sigma(r)=0.6$ and for \textit{no repeat} is $\sigma(\neg r)=0.4$ ($\sigma(r) > \sigma(\neg r)$), hence, the robot selects to repeat the test. It is worth noticing that the framework had two arguments for repeat and one for no repeat with the same base score, and then the decision intuitively seems coherent. Now, the care recipient performs the Standing Balance test and loses the balance in the \textit{together} position, indicating a possible risk of falling, which activates T3 and T1. The former supports the latter, which attacks ``repeat'' and supports ``not repeat''. The final strength for repeating is $\sigma(r)=0.495$ while for not repeating, it is $\sigma(\neg r)=0.505$ ($\sigma(\neg r) > \sigma(r)$), leading the robot to choose not to repeat the following tests. This example is represented graphically in Fig.~\ref{fig:example_frailty}.

        \end{example}

        The second example focuses on the modification of existing arguments when their importance shifts.
        Indeed, in some cases, users might want to rectify their preferences or previous arguments, or the robot’s new observations might alter the importance of certain arguments with a corresponding change in their base score. Those changes can be rapidly inserted into the framework by modifying the arguments' base score. The example below shows this adaptation (satisfying Property~\ref{property:base_score_discrimination}): 

        \begin{example}
             
            The caregiver prefers not to repeat the tests due to a packed schedule of visits for the day, so they need to act quickly (CG1 activated). Since this decision is important to the caregiver, the base score for this argument is set at 0.9. On the other hand, the care recipient prefers to repeat the tests, if needed, to ensure the robot accurately measures their frailty (CR2 activated). Given the care recipient's concern, the base score for their argument is set at 0.7.
            The final strength for \textit{repeat} is $\sigma(r)=0.48$, and for \textit{no repeat} is $\sigma(\neg r)=0.52$ ($\sigma(\neg r) > \sigma(r)$), then the robot decides not to repeat the tests if necessary. During the morning, the caregiver receives some calls from the patients cancelling the visits. After checking the waiting room, the caregiver confirms that there are not as many patients as expected. As a result, the caregiver lowers the base score of the CG1 argument to 0.6. The final strengths are now: $\sigma(r)=0.505$ and $\sigma(\neg r)=0.495$ ($\sigma(r) > \sigma(\neg r)$), the robot now selects to repeat the tests if necessary.

        \end{example}

        In summary, the previous examples demonstrate that the argumentation framework quickly adjusts to changes from new arguments or modifications to the base scores of existing ones, without requiring retraining.

    \subsection{Framework evaluation and sensitivity analysis}
    \label{sec:framework_evaluation_and_sensitivity_analysis}
    
    This section analyses the framework's behaviour under varying inputs and configurations, focusing solely on scenarios with conflicting preferences. The arguments considered in this section are those shown in Fig.~\ref{fig:framework_frailty}.
    We define two scenarios.
    In scenario (1), the caregiver opposes repeated actions while the care recipient supports them. There are a total of 7 activatable arguments, including the risk of falling. The total amount of possible combinations with activated or deactivated arguments is 128 ($2^7$). 
    In scenario (2), those preferences are reversed. Here, there are 18 arguments, which make a total of 262144 ($2^{18}$) combinations. Additionally, assigning a base score to each argument makes the number of possible combinations infinite. For this scenario, all arguments have been assigned the same base score of 0.5.

    \begin{table}[t]
        \centering
        \resizebox{0.95\columnwidth}{!}{%
        \begin{tabular}{|c|c|c|c|c|}
        \hline
        \textbf{Scenario} & \textbf{Combinations} & \textbf{R (\%)} & \textbf{N-R (\%)} & \textbf{Tie (\%)} \\
        \hline
        \multirow{3}{*}{1} & All combinations (n=128) & 1.6 & 92.2 & 6.2 \\
        & T1 activated (n=120) & 0.0 & 96.7 & 3.3 \\
        & Without risk of falling (n=4) & 25.0 & 25.0 & 50.0 \\
        \hline
        \multirow{3}{*}{2} & All combinations (n=262144) & 13.4 & 83.2 & 3.4 \\
        & T1 activated (n=245760) & 11.7 & 85.4 & 2.9 \\
        & Without risk of falling (n=8192) & 39.2 & 49.5 & 11.3 \\
        \hline
        \end{tabular}
        }
        \caption{Distribution of preference selections for  Repeat (R), No Repeat (N-R) and Tie across different argument configurations for the two conflicting scenarios. The first and fourth rows correspond to all the possible argument combinations. The second and fifth rows are the combinations where the argument T1 is activated. The third and last rows are the combinations when the risk of falling arguments (the upper part of the AF) are removed. All the arguments' base score is 0.5. R means Repeat, and N-R means No Repeat.}
        \label{table:influence_risk_agent}
    \end{table}

    Table~\ref{table:influence_risk_agent} presents the selection percentages for various combinations of activated and deactivated arguments across both scenarios. The first row lists all possible combinations, the second row focuses on instances when the risk of falling argument (T1) is activated, and the third row shows the distributions when the risk of falling arguments are not included. In this last scenario, while the argument CR1 remains in the framework, it does not influence the final strengths because there is no path connecting it to the \textit{Option Arguments}. 
    From this table, we can observe the impact of the risk of falling argument (T1), which skews the distributions towards ``no repeat'' thereby making the assessment safer based on actual observations of the robot. This effect is particularly noticeable when comparing the values between the first and second rows for each scenario. In contrast, the rows that exclude the risk of falling arguments show a more balanced distribution between repeating and not repeating. This balance is expected since the \textit{Option Arguments} have an equal number of attackers and supporters in both scenarios. The differences observed in the last row regarding repetition are attributable to the indirect attackers and supporters of the \textit{Option Arguments}.
 
    Additionally, it is possible to compute the influence of the relations between arguments on the \textit{Option Arguments}, based on the approach from~\cite{Yin_ijcai24} named \textit{Relation Attribution Explanations (RAEs)}. This analysis shows the influence of each relation on the \textit{Option Arguments}. Table~\ref{table:influence_final_decision} shows an example of influences for scenario (2), since it contains more arguments and is more complete, with all the arguments activated and base scores at 0.5.

    \begin{table}[t]
        \centering
        \resizebox{0.8\columnwidth}{!}{%
            \begin{tabular}{|c|c|c|}
            \hline
            
             \textbf{Relation (type, from, to)} & \textbf{Contribution to $r$} & \textbf{Contribution to $\neg r$} \\
            
            \hline
            (att, T1, $r$) & -0.15 & 0 \\
            \hline
            (att, CR1, T1) & 0.01 & -0.01 \\
            \hline
            (att, CR3, $r$) & -0.14 & 0 \\
            \hline
            (att, CG2, CR5) & 0.00 & -0.00 \\
            \hline
            (att, CR6, $r$) & -0.12 & 0 \\
            \hline
            (att, CG3, $\neg r$) & 0 & -0.12 \\
            \hline
            (att, CR7, $r$) & -0.11 & 0 \\
            \hline
            (att, T8, CR7) & 0.01 & -0.01 \\
            \hline
            (att, CG4, $\neg r$) & 0 & -0.11 \\
            \hline
            (att, CG5, $\neg r$) & 0 & -0.11 \\
            \hline
            (sup, T1, $\neg r$) & 0 & 0.15 \\
            \hline
            (sup, T2, T1) & -0.01 & 0.01 \\
            \hline
            (sup, T3, T1) & -0.01 & 0.01 \\
            \hline
            (sup, T4, T1) & -0.01 & 0.01 \\
            \hline
            (sup, T5, T1) & -0.01 & 0.01 \\
            \hline
            (sup, CR3, $\neg r$) & 0 & 0.14 \\
            \hline
            (sup, CR4, CR3) & -0.01 & 0.01 \\
            \hline
            (sup, CR5, CR3) & -0.01 & 0.01 \\
            \hline
            (sup, CR6,  $\neg r$) & 0 & 0.12 \\
            \hline
            (sup, T6, CR6) & -0.01 & 0.01 \\
            \hline
            (sup, CG3, $r$) & 0.12 & 0 \\
            \hline
            (sup, T7, CG3) & 0.01 & -0.01 \\
            \hline
            (sup, CR7,  $\neg r$) & 0 & 0.11 \\
            \hline
            (sup, T2, CR7) & -0.01 & 0.01 \\
            \hline
            (sup, CG4, $r$) & 0.11 & 0 \\
            \hline
            (sup, CG5, $r$) & 0.11 & 0 \\
            \hline
            \textbf{Final strength} & \textbf{$r$: 0.23} & \textbf{$\neg r$: 0.77} \\
            
            \hline
            \end{tabular}
        }
        \caption{Scenario 2 (caregiver for repeat, care recipient against repeat). Approximate influence of each relation between arguments on the final \textit{Option Arguments} when all arguments are active and with a base score of 0.5. The value 0 is for strict zeros, and 0.00 refers to values $< 5$ · $10^{-3}$.}
        \label{table:influence_final_decision}
    \end{table}

    The results show interesting findings. Using the Quadratic Energy model, the arguments that directly attack or support the \textit{Option Arguments} have a much higher influence on the decisions than the ones that indirectly attack or support them. The attack of the argument CG2 to CR5, which is at 3 argument distance from the \textit{Option Arguments}, has an almost negligible influence. This suggests that adding direct attacks or supports, or incrementing the base score of the closest arguments to the option ones, would be much more effective than doing so on further arguments. The robot can use this information to develop a strategy for generating explanations behind the rationale of its decision, making its decision more understandable to the user, eventually impacting on how they will provide their arguments in the next interaction.

    
    \begin{figure}[t]
        \centering
        \begin{subfigure}{0.48\columnwidth}
            \includegraphics[width=\linewidth]{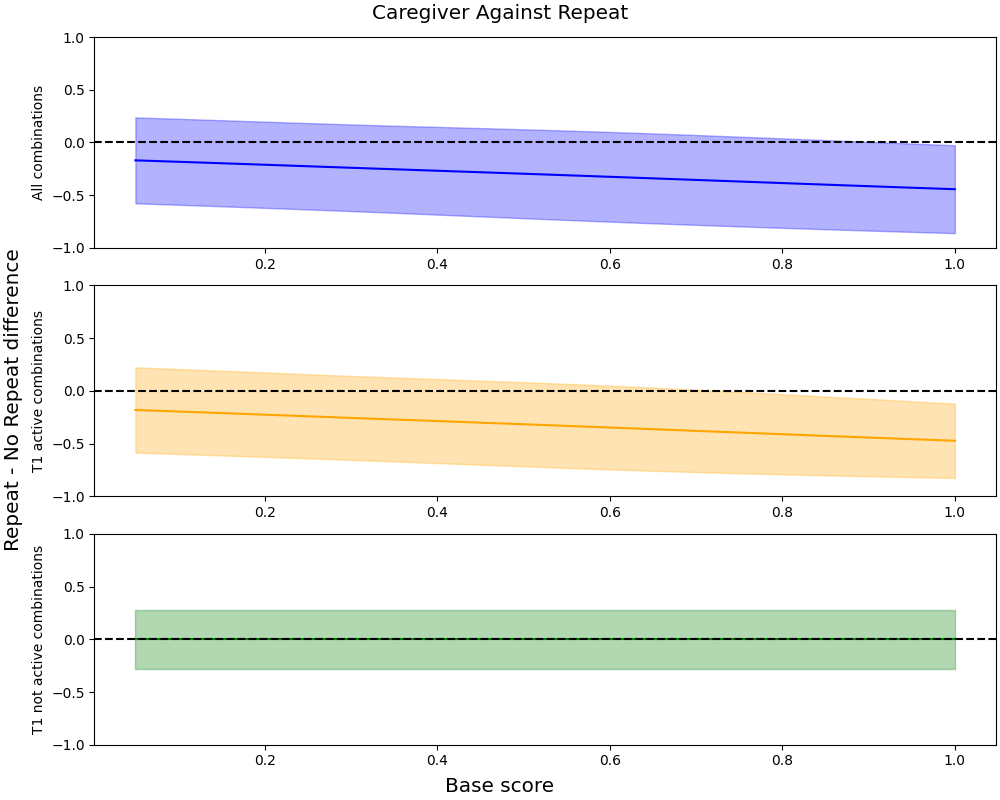}
            \caption{Mean and standard deviation (95\%) distribution of the difference between the final strengths of repeat and no repeat, ($\sigma (r)-\sigma (\neg r)$), for all the combinations in scenario (1). The top plot corresponds to the totality of the combinations (activated and deactivated), the second plot corresponds to the combinations where T1 is activated, and the third one corresponds to the combinations where T1 is deactivated.}
            \label{fig:sensitivity_D1}
        \end{subfigure}
        \hfill
        \begin{subfigure}{0.48\columnwidth}
            \includegraphics[width=\linewidth]{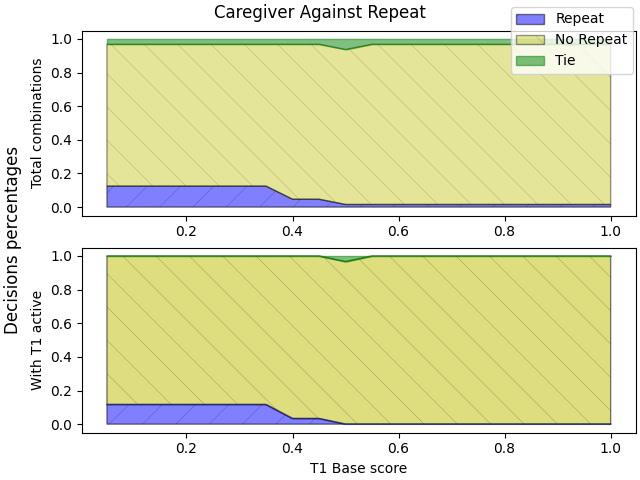}
            \caption{Percentages of the decisions, in scenario (1), where the caregiver is against repeat. The distribution of all combinations is plotted at the top. At the bot, the combinations where the risk of falling argument is active (T1). The blue dashed section refers to the percentages that the robot decided to repeat the tests. The yellow dashed section refers to not repeating. The green area refers to ties.}
            \label{fig:sensitivity_D1_percentages}
        \end{subfigure}

        \caption{Sensitivity analysis of the T1 argument's base score.}
        \label{Bigfig}
        
    \end{figure}

    Finally, we conducted a sensitivity analysis to evaluate how the base score of a critical argument influences the final decision. We varied the base score of the ``risk of falling'' argument (T1), which acts as a parameter to make the robot's decision-making more or less conservative.
    Fig.~\ref{fig:sensitivity_D1} illustrates the results, showing the difference in final strength between the repeat and no repeat options ($\sigma(r)-\sigma(\neg r)$) as the base score of  T1 is increased from 0 to 1 (all other arguments were held at a base score of 0.5). The plot shows a clear trend: as the T1 base score increases, the strength difference in favour of repeating diminishes. This directly impacts the robot's choices, as shown in Fig.~\ref{fig:sensitivity_D1_percentages}. The selection percentage for repeat decreases as the T1 base score grows. This confirms that a higher base score for the ``risk of falling'' argument makes the robot more cautious, eventually reaching a point where it will strictly select the no repeat option when this risk is present.

\section{Limitations and outlook}
    The presented system has proven to be a powerful tool for robots to resolve preference conflicts in HRI contexts. Nonetheless, it is important to acknowledge certain limitations that should be addressed in future work. We have categorised these limitations into two groups: those related to the design of the framework and those pertaining to the interactions that occur while using the system.

    Regarding the AF design, the following limitations are noted: firstly, this work does not implement Argument Mining, which is the field in CA that generates and relates arguments from user input. The arguments in the AF presented in this work were decided through co-design with healthcare professionals. However, in a complete and fully autonomous real-world deployment, Argument Mining will be necessary for getting the users and task arguments. This will become an important challenge for future work, which can be potentially overcome with the usage of LLM~\cite{Gorur_arxiv24}. Secondly, all the base scores in Sec.~\ref{sec:robotic_frailty_assessment} are set to 0.5, except in the sensitivity analysis. In some scenarios, it may seem reasonable for the system to assign different base scores to users' arguments. For instance, users might give arguments that are considered very relevant and important to them~\cite{Battaglia_icsum24,Civit_aamas26}. Those arguments could have a greater base score. Another factor considered when setting the base scores can be the expertise of the users in the context, e.g. in assistive scenarios, health-related arguments given by the caregiver may be more grounded than those given by the care recipient. While this can improve decision quality, it also raises a risk of systemic bias, prioritising caregiver preferences at the expense of the care recipient's autonomy. Thirdly, while this study assumes distinct arguments for each user, scenarios where users share arguments but assign different base scores require further exploration. Methods to aggregate individual AFs and base scores into a general AF should be further explored~\cite{Dickie_AAMAS25}.

    Respecting the system's interaction limitations, we mention the following: firstly, no explanation is generated regarding the reason for the robot's selection. Explanations could be based on the influence of each argument or their relationships to the final decision (see Sec.~\ref{sec:framework_evaluation_and_sensitivity_analysis}). Similarly, if a user complains or questions the reason behind the robot's decision, the robot must provide a clear and understandable explanation to facilitate the necessary feedback for adjusting its decision. To achieve this, it is essential to identify the necessary and sufficient changes and explanations required within the framework ~\cite{Kampik_IJAR24}. Secondly, this work is presented in a particular assistive use case. Even though it is generalisable for any scenario where multiple users have competing preferences, it should be tested and studied in different use cases and in a real-world scenario. Additionally, this work is based on a single-interaction approach, in which a care recipient interacts with the robot once, only to perform the tests. More investigation on how this system behaves in long-term interactions is needed.

    Future directions include extensions of this framework, such as selecting between gradual preferences, which could be obtained by defining a rule based on the final strengths of the \textit{Option Arguments}. Additionally, a safety layer in charge of detecting possible users manipulating the framework outputs could be developed. Since HRI scenarios are mostly user-centred, the users will probably demand explanations and challenge the robot's decision. The robot should be capable of providing those explanations and adapting to the user's feedback, making the system contestable~\cite{Leofante_icpkrr24, Yin_arxiv25}. This feedback can be in the form of new arguments or corrections to the decision. Tools to interpret that feedback and modify or fine-tune the AF will be needed.

\section{Conclusions}


    The literature on personalisation in HRI has mainly focused on adapting the robot's behaviour according to the preferences of a single user. However, when multiple users are involved, their preferences may conflict, leading to complex decision-making scenarios. This challenge has received limited attention in the literature, largely due to the inherent difficulty of identifying satisfactory strategies for resolving such conflicts. 

    This work makes a significant advance towards this gap. We present a novel argumentation framework for multi-user personalisation, the MUP-QBAF, which formally represents the problem, and a novel algorithm which allows solving it. An advantage of using MUP-QBAF in an HRI context is that the AF is generated from users' arguments for their preferences, and from robot observations in the form of arguments, adapting the decision to the users' preferences and the context. The system allows for considering positive, neutral, and negative preferences for different options, overcoming the limitation in HRI of focusing only on the positive ones. The properties and capabilities of the MUP-QBAF in HRI are developed within a practical use case, in which an assistive robot must decide between a care recipient's and a caregiver's preferences for repeating a test or not when performing frailty assessments to older adults. In this example, it is shown that adding new arguments or modifying the base score of existing ones can change the framework's output, adapting to new contexts.
    
    The framework representation as a simple graph allows 
    allows for the rapid introduction of new arguments, facilitating the learning process without the need for retraining, unlike conventional models. This easy adaptation is a potentially useful tool for long-term interactions in HRI. 
    Additionally, it allows for a more transparent way of providing explanations for the robot's decisions.
    Overall, the MUP-QBAF framework provides a significant step toward more adaptive, contestable, and user-centred HRI in multi-user environments.




\section*{Acknowledgment}



\bibliography{sn-bibliography}

\end{document}